\def\draft{0}

\documentclass[11pt]{article}
\usepackage{amsmath}
\usepackage{amssymb}
\usepackage{amsthm}
\usepackage{mathtools}
\usepackage[dvipsnames]{xcolor}
\usepackage{mdframed}
\usepackage{algorithm}
\usepackage{algorithmic}

\newcommand{\Paren}[1]{\left(#1\right)}

\newcommand{\abs}[1]{\lvert#1\rvert}

\newcommand{\sset}[1]{\{#1\}}
\newcommand{\Set}[1]{\left\{#1\right\}}

\newcommand{\norm}[1]{\lVert#1\rVert}
\newcommand{\Norm}[1]{\left\lVert#1\right\rVert}

\newcommand{\iprod}[1]{\langle#1\rangle}

\newcommand{\Esymb}{\mathbb{E}}

\DeclareMathOperator*{\E}{\Esymb}

\newcommand\bdot\bullet

\DeclareMathOperator{\Tr}{Tr}

\newcommand{\R}{\mathbb R}

\newcommand{\cA}{\mathcal A}
\newcommand{\cB}{\mathcal B}

\newcommand{\cS}{\mathcal S}

\newcommand{\bbP}{\mathbb P}

\renewcommand{\leq}{\leqslant}
\renewcommand{\le}{\leqslant}
\renewcommand{\geq}{\geqslant}
\renewcommand{\ge}{\geqslant}

\let\epsilon=\varepsilon
\numberwithin{equation}{section}
\newcommand\MYcurrentlabel{xxx}
\newcommand{\MYstore}[2]{%
  \global\expandafter \def \csname MYMEMORY #1 \endcsname{#2}%
}
\newcommand{\MYload}[1]{%
  \csname MYMEMORY #1 \endcsname%
}
\newcommand{\MYnewlabel}[1]{%
  \renewcommand\MYcurrentlabel{#1}%
  \MYoldlabel{#1}%
}
\newcommand{\MYdummylabel}[1]{}
\newcommand{\torestate}[1]{%
  \let\MYoldlabel\label%
  \let\label\MYnewlabel%
  #1%
  \MYstore{\MYcurrentlabel}{#1}%
  \let\label\MYoldlabel%
}
\newcommand{\restatedef}[1]{%
  \let\MYoldlabel\label
  \let\label\MYdummylabel
  \begin{definition*}[Restatement of \cref{#1}]
    \MYload{#1}
  \end{definition*}
  \let\label\MYoldlabel
}
\newcommand{\restatetheorem}[1]{%
  \let\MYoldlabel\label
  \let\label\MYdummylabel
  \begin{theorem*}[Restatement of \cref{#1}]
    \MYload{#1}
  \end{theorem*}
  \let\label\MYoldlabel
}
\newcommand{\restatelemma}[1]{%
  \let\MYoldlabel\label
  \let\label\MYdummylabel
  \begin{lemma*}[Restatement of \cref{#1}]
    \MYload{#1}
  \end{lemma*}
  \let\label\MYoldlabel
}
\newcommand{\restateprop}[1]{%
  \let\MYoldlabel\label
  \let\label\MYdummylabel
  \begin{proposition*}[Restatement of \cref{#1}]
    \MYload{#1}
  \end{proposition*}
  \let\label\MYoldlabel
}
\newcommand{\restatefact}[1]{%
  \let\MYoldlabel\label
  \let\label\MYdummylabel
  \begin{fact*}[Restatement of \cref{#1}]
    \MYload{#1}
  \end{fact*}
  \let\label\MYoldlabel
}
\newcommand{\restate}[1]{%
  \let\MYoldlabel\label
  \let\label\MYdummylabel
  \MYload{#1}
  \let\label\MYoldlabel
}

\newcommand{\e}{\epsilon}
\newcommand{\eps}{\epsilon}

\allowdisplaybreaks
\newcommand*{\Id}{\mathrm{I}}

\newcommand*{\normf}[1]{\norm{#1}_{\mathrm{F}}}

\newtheorem{theorem}{Theorem}[section]

\newtheorem{lemma}[theorem]{Lemma}

\definecolor{niceish}{HTML}{74b807} 
\ifnum\draft=1
\newcommand{\siuon}[1]{\textcolor{teal}{[Siu On: #1]}}
\newcommand{\lucas}[1]{\textcolor{violet}{[Lucas: #1]}}
\newcommand{\tom}[1]{\textcolor{WildStrawberry}{[Tommaso: #1]}}
\newcommand{\gleb}[1]{\textcolor{WildStrawberry}{[Gleb: #1]}}
\else
\newcommand{\siuon}[1]{}
\newcommand{\lucas}[1]{}
\newcommand{\gleb}[1]{}
\newcommand{\tom}[1]{}
\fi

\newcommand{\om}{\om}

\RequirePackage[outline]{contour} 
\contourlength{0.065pt}
\contournumber{10}%

\newcommand{\Schat}{\mathrm{S}}

\usepackage{hyperref}

\begin{document}

\title{On Purely Private Covariance Estimation\gleb{IF YOU SEE IT, IT IS A DRAFT VERSION}}

\author{
  Tommaso d'Orsi%
  \thanks{Bocconi University, Italy.}
  \and
  Gleb Novikov%
  \thanks{Lucerne School of Computer Science and Information Technology, Switzerland.}
}



\maketitle

\begin{abstract}
    We present a simple perturbation mechanism for the release of $d$-dimensional covariance matrices $\Sigma$ under pure differential privacy. For  large datasets with at least $n\geq d^2/\varepsilon$ elements, our mechanism recovers the provably optimal Frobenius norm error guarantees  of \cite{nikolov2023private}, while simultaneously achieving best known error for all other $p$-Schatten norms, with $p\in [1,\infty]$. Our error is information-theoretically optimal for all $p\ge 2$, in particular, our mechanism is the first purely private covariance estimator that achieves optimal error in spectral norm.
    
    For small datasets $n< d^2/\varepsilon$, we further show that by projecting the output onto the nuclear norm ball of appropriate radius, our algorithm achieves the optimal Frobenius norm error $O(\sqrt{d\;\text{Tr}(\Sigma) /n})$, improving over the known bounds of $O(\sqrt{d/n})$ of \cite{nikolov2023private} and ${O}\big(d^{3/4}\sqrt{\text{Tr}(\Sigma)/n}\big)$ of \cite{dong2022differentially}.
\end{abstract}

\clearpage
\tableofcontents{}
\clearpage

\pagestyle{plain}
\setcounter{page}{1}

\section{Introduction}\label{sec:introduction}
For a set of vectors $X=\Set{x_1,\ldots,x_n}\subset \cB_d$, where $\cB_d$ is a $d$-dimensional Euclidean unit ball\footnote{$X\subset \cB_d$ is a standard assumption in the literature on private covariance estimation.}, the covariance matrix is defined as\footnote{We work with the non-centered covariance. All our results are also valid for estimation of the centered covariance  $\tfrac1n\sum_{i=1}^n x_i x_i^\top - \Paren{\tfrac{1}{n}\sum_i x_i}\Paren{\tfrac{1}{n}\sum_i x_i}^\top$ if the parameters of the algorithms and error bounds are adjusted by a small absolute constant factor.}
\begin{align*}
    \Sigma := 
    \tfrac1n\sum_{i=1}^n x_i x_i^\top\,.
\end{align*}
Because of its widespread use in training prediction models, the task of releasing a  private version of $\Sigma$ has received significant attention (see \cite{amin2019differentially, dong2022differentially,nikolov2023private, cohen2024perturb} and references therein). 
In contrast to ad-hoc private training methods, the advantage of working with a synthetic version $\hat{\Sigma}$ of the covariance is that one can replace $\Sigma$ with $\hat{\Sigma}$ and directly run standard training algorithms. This approach  not only fits seamlessly in existing training pipelines, but typically also leads to lower privacy costs, as these are paid only once, upon computing the private estimate.

The de facto standard for privacy is differential privacy (\cite{dwork2006calibrating}). A randomized algorithm $\cA$ is said to be $(\eps,\delta)$-differentially private if, for any pair of datasets $X,X'\subseteq \R^d$ differing in at most $1$ element, and for any measurable event $S$ in the range of $\cA,$  it holds
\begin{align*}
    \bbP\Paren{\cA(X)\in S}\leq e^{\eps}\cdot \bbP\Paren{\cA(X')\in S} +\delta.
\end{align*}
The setting $\delta=0$ is usually refer to as \textit{pure} privacy, while the setting $\delta>0$  is called approximate differential privacy.
Pure privacy is the strictest notion of privacy, it disallows any complete privacy breach and naturally offers privacy protection to groups of elements in the datasets, as opposed to just individuals (see \cite{nikolov2023private} for a more comprehensive discussion). 

In the context of private covariance estimation, the difference between pure and approximate differential privacy is surprisingly stark.
For $\delta>0,$ efficient algorithms can achieve accuracy $\normf{\hat\Sigma-\Sigma}^2\leq O\Paren{\sqrt{\log (1/\delta)}/\eps}\cdot\min\Set{d/n,\,d^{1/4}\sqrt{\eps/n}}$ and this bound is known to be tight (\cite{nikolov2013geometry,dwork2015efficient,kasiviswanathan2010price}).
In contrast, pure privacy lower bounds for the special case of $2$-way marginals already rule out a Frobenius-norm error smaller than $\min\Set{d^{3/2}/\eps n, \sqrt{d/n\eps}}$ (\cite{hardt2010geometry, nikolov2023private}). 
Here the second term is smaller for small sample sets $n< d^2/\e$.

\cite{nikolov2023private} introduced a variant of the projection mechanism, based on the Johnson-Lindestrauss transform (\cite{johnson1984extensions}) matching these bounds up to constant factors.
For datasets with points of small average Euclidean norm,  \cite{dong2022differentially} designed a pure differentially private algorithm achieving a Frobenius norm error of $\Tilde{O}(d^{3/4}\sqrt{\Tr\Sigma/n}+\sqrt{d}/n),$ thus going beyond \cite{nikolov2023private} for the special case  of matrices satisfying $\Tr\Sigma \ll O(1/\eps)\cdot \min\Set{1/n,\, d^{-1/4}}.$ 

In this work, we show that  particularly simple perturbation and projection mechanisms can be used to tie together these two upper bounds, and improve over them. 
Our analysis not only tightens the Frobenius norm bound, accounting for the structure of the covariance matrix, but further shows the error can be controlled in a stronger sense: bounding  the spectral discrepancy simultaneously in \textit{every} Schatten norm $\norm{\cdot}_{\textnormal{S}_p}$ with $p\in [1,\infty].$

\paragraph{Notation.}
For $p\in[1,\infty]$ the $p$-Schatten norm of $A$ is defined as $\Norm{A}_{S_p}=\Paren{\sum_i \sigma_i^p}^{1/p}$, where $\sigma_1\geq\ldots\geq\sigma_d\geq0$ are singular values of $A$. For simplicity we denote the spectral norm $(p=\infty)$ with $\Norm{A},$ the Frobenius norm $(p=2)$ with $\Norm{A}_{\textnormal{F}}$ and the nuclear norm $(p=1)$ with $\Norm{A}_*.$

\subsection{Results}
Thanks to the simplicity of our mechanisms, we can directly introduce them and states their guarantees. 
Our first result shows that the following  additive perturbation mechanism yields tight error bounds for all Schatten norms.
\begin{algorithm}[t]
\caption{$\varepsilon$-Differentially Private Covariance Estimation via Perturbations}
\label{alg:main-schatten}

\textbf{Input:} $\Sigma$, $\varepsilon$

\begin{enumerate}
  \item Return $\Sigma + Z$ with $Z$ drawn from the nuclear-Laplace law 
  $p(Z) \propto \exp\!\big(-\|Z\|_* / \rho\big)$, where $\rho = 2 / (\varepsilon n)$.
\end{enumerate}

\end{algorithm}

Efficient and streamlined samplers for the nuclear-Laplace law are easy to construct and can be seen as a specific instantiation of the $K$-norm mechanism  of \cite{hardt2010geometry}. Therefore the algorithm runs in polynomial time.
\begin{theorem}\label{thm:main-schatten-norm}
    Let $\eps>0$. Algorithm~\ref{alg:main-schatten} is $\eps$-DP, runs in polynomial time, and with high probability its output $\hat{\Sigma}$ satisfies
    \begin{align*}
        \Norm{\hat{\Sigma}-\Sigma}_{\textnormal{S}_p}\leq {\frac{3\cdot  d^{\,1+1/p}}{\eps n}}\qquad 
    \end{align*}
for every $p\in [1,\infty]$.
\end{theorem}
Theorem~\ref{thm:main-schatten-norm} guarantees bounded error \textit{simultaneously} for all Schatten norms, including the nuclear norm ($p=1$) and the spectral norm ($p=\infty$), implying spectral discrepancies between the true covariance and the pure differentially private estimate are small under any natural reweighing. For large datasets $n\geq d^2/\e$ --typical of real-world applications-- the Theorem recovers the guarantees of \cite{nikolov2023private} for the Frobenius norm error. Furthermore, as for $p\ge2$
\begin{align}\label{eq:schatten-norm-distortion}
    \Norm{\hat{\Sigma}-\Sigma}_{\textnormal{S}_p}\geq d^{1/p-1/2}\Norm{\hat{\Sigma}-\Sigma}_{\textnormal{F}},
\end{align}
the error of Theorem~\ref{thm:main-schatten-norm} is \textit{optimal} for $p\ge 2$ in the large sample regime $n \ge d^2/\e$. Indeed, the information-theoretic lower bound (\cite{nikolov2023private}) shows that if $n\ge d^2/\e$, no pure DP estimator can achieve error $o(d^{3/2}/\varepsilon n)$ in Frobenius norm, and hence no estimator can achieve error $o(d^{1+1/p}/\varepsilon n)$ in Schatten $p$-norm for $p\ge 2$. In particular, our spectral norm bound $O(d/\eps n)$ is information-theoretically optimal in the large sample regime $n\ge d^2 / \e$.

\bigskip
Our second result shows that the Frobenius norm error can be improved, projecting the perturbed covariance obtained from Algorithm~\ref{alg:main-schatten} back onto the nuclear ball of radius $\Norm{\Sigma}_*$, up to a tiny perturbation required to maintain privacy. 

\begin{algorithm}[t]
\caption{$\varepsilon$-Differentially Private Covariance Estimation via Projection}
\label{alg:projection}

\textbf{Input:} $\Sigma$, $\varepsilon$

\begin{enumerate}
  \item Run Algorithm~\ref{alg:main-schatten} on $\Sigma$, $\varepsilon/2$. Let $\hat{\Sigma}$ be its output.
  \item Return the orthogonal projection $\Tilde{\Sigma}$ of $\hat{\Sigma}$ onto the nuclear norm ball
  \[
    \bigl\{\, Y \in \mathbb{R}^{d\times d} \; \big|\;
      \|Y\|_* \le 
      \max\!\bigl( 0,\; 2\|\Sigma\|_* + \mathrm{Lap}\!\bigl(\tfrac{10}{\varepsilon n}\bigr) \bigr)
    \,\bigr\}.
  \]
\end{enumerate}

\end{algorithm}


\begin{theorem}\label{thm:main-frobenius}
    Let $\eps>0$. Algorithm~\ref{alg:projection} is $\eps$-DP, runs in polynomial time, and with probability $1-d^{-10}$ its output $\tilde{\Sigma}$ satisfies
    \begin{align*}
        \Norm{\tilde{\Sigma}-\Sigma}_{\textnormal{F}}\leq O\Paren{1}\cdot\min\Set{\frac{d^{3/2}}{\eps n},\, \sqrt{\frac{d\cdot \Tr\Sigma}{\eps n}}+\frac{\sqrt{d\log d}}{\eps n}}.
    \end{align*}
\end{theorem}
Theorem~\ref{thm:main-frobenius} always recovers the result of \cite{nikolov2023private}, but improves over it whenever $\Tr\Sigma < d^2/n.$ 
As values of the trace of the covariance matrix can be significantly smaller than $1$, Theorem~\ref{thm:main-frobenius} can lead to important accuracy improvements both in the small sample regime $n\le d^2/\e$ where the improvement is of the order $O(1/\Tr\Sigma)$, and in the large sample regime. 
Theorem~\ref{thm:main-frobenius} also strictly improves over the result of \cite{dong2022differentially} (that also captured the dependence on $\Tr(\Sigma)$) in all regimes by a $d^{1/4}$ factor.

\section{Techniques}\label{sec:techniques}
To illustrate our ideas we start by reviewing the algorithm of \cite{nikolov2023private}.
On a high level, this consists of three steps: (1) project  the data onto a low dimensional space, (2) add noise via a $K$-norm mechanism, and (3) project the result back onto the space of covariance matrices (the intersection of the nuclear unit ball and the positive semidefinite cone).
As the  cost of traditional perturbation mechanisms tends to be proportional to the ambient dimension, the goal of the first step is to reduce this privacy cost by embedding the dataset into a small subspace, thus paying a cost proportional only to the effective dimension of the data.
The result $\Tilde{\Sigma}$ of steps (1) and (2) can then be shown to be private and close to (the projection of) $\Sigma$ up to a Frobenius distance of order $O(d^{3/2}/n\eps)$. The third step is an orthogonal projection onto a convex set containing the ground truth, thus it cannot increase the $\ell_2$-distance to $\Sigma$, but, remarkably, for small datasets can provably shrink it to $O(\sqrt{d/n\eps}).$

Despite its utility in terms of the Frobenius norm error, this algorithm cannot be expected to be accurate in other Schatten norms, in a strong sense. The dimensionality reduction step rely on a Johnson-Lindestrauss transform, but such mappings are known \textit{not} to exist for spaces that are far from being Euclidean and, more specifically, for $p$-norm space with $p\in [1,\infty]\setminus\sset{2}$ (\cite{brinkman2005impossibility, lee2005metric, johnson2010johnson}).

\paragraph{The perturbation mechanism.}
To bypass this inherent limitation, we do not attempt to explicitly embed the dataset in a low-dimensional space and instead carefully craft the additive noise to inject.
Since adjacent covariance matrices $\Sigma, \Sigma'$ satisfy $\norm{\Sigma-\Sigma'}_* \le 2/n$, the $K$-norm mechanism (\cite{hardt2010geometry}) with density $p(Z)$ proportional to $\exp(-\eps n \norm{Z}_*/2)$ is $\varepsilon$-DP. Note that it is possible to sample from such distribution efficiently, leveraging any efficient weak separation oracle for the nuclear norm\footnote{See Lemma A.2. in \cite{hardt2010geometry} for more details.}. Hence it only remains to bound Schatten norms of $Z$ sampled from this distribution, and in fact, it is enough to bound $\norm{Z}$ by $O(\frac{d}{\eps n})$ since $\norm{Z}_{\textnormal{S}_p} \le \norm{Z}\cdot d^{1/p}$.

Consider a singular value decomposition  of $Z = U D V^\top$. By rotational symmetry, $U, D, V$ are mutually independent. For real $d\times d$ matrices, the Lebesgue measure $dZ$ factors under this map
\[
  \Phi:\ (U,\sigma,V)\ \mapsto\ U\,\mathrm{diag}(\sigma)\,V^\top,\qquad \sigma=(\sigma_1,\dots,\sigma_d)\in\mathbb{R}^d_{\ge 0},
\]
as\footnote{This factorization is standard, and can be found, in particular, in Lemma 1.5.3 from \cite{chikuse2003statistics}.}
\[
  dZ\ =\ C_d\ \prod_{1\le i<j\le d}\abs{\sigma_i^2-\sigma_j^2}\,d\sigma\ \,d\mu_{\mathrm{Haar}}(U)\ \,d\mu_{\mathrm{Haar}}(V),
\]
for a (dimension-dependent) constant $C_d>0$, where $d\sigma$ is the Lebesgue measure on $\mathbb{R}^d$, and $d\mu_{\mathrm{Haar}}$ denotes the Haar probability measure on the orthogonal group. 

Let $R=\sum_i \sigma_i$ be the \emph{radial} component and $w_i=\sigma_i/R$ be  \emph{relative weights} of the singular values. Then $d\sigma=R^{\,d-1}\, dR\, d\lambda_{\Delta}(w)$, where $d\lambda_{\Delta}(w)$ is the $(d-1)$-dimensional Lebesgue measure on the simplex $\sum_{i=1}^d w_i = 1\,, w_i > 0$.
Substituting $\sigma_i=Rw_i$ yields the joint density
\[
p(U,R,w,V)\ \propto\ \underbrace{e^{-R/\rho}\,R^{\,d-1}\,R^{d\cdot(d-1)}}_{=:\,f(R)}\;
\cdot\;\underbrace{\prod_{i<j}\lvert w_i^2-w_j^2\rvert}_{=:g(w)}\;
\cdot\; d\mu_{\mathrm{Haar}}(U)\, d\mu_{\mathrm{Haar}}(V)\, dR\, d\lambda_{\Delta}(w)\,.
\]

This implies that to show the desired concentration of the spectral norm, it is enough to study the product of two independent random variables: $R$ and $\max_{i\in [d]} w_i$. It is well-known that the distribution of $R$ is\footnote{See, for example, Remark 4.2 in \cite{hardt2010geometry}.} $\mathrm{Gamma}(d^2,\rho)$, and it is well-concentrated around $\rho d^2$ (recall that $\rho = 2/(\eps n)$).
The result would then follow if we could show that with high probability $\max_{i\in [d]} w_i\leq 1/d+\tilde{O}(1/d^{3/2})$ as by union bound $\Norm{Z}\leq O(d^2/\eps\cdot n)\cdot O(1/d)\leq O(\tfrac{d}{\eps n}).$

\paragraph{Bounding the singular values of $Z$.}
The analysis of the maximal entry of $w$ turns out to be fairly delicate. Note that $w$ is a vector with density proportional to $e^{-V(w)}:=\prod_{i<j}\lvert w_i^2-w_j^2\rvert$, restricted to the $d$-dimensional simplex $\sum_i w_i=1$, $w_i>0$. 
Our approach consists of showing that this distribution is $\kappa$-strongly log-concave. This allows us to leverage the entry-wise sub-Gaussian concentration:
\begin{align}\label{eq:concentration-log-concave}
    \mathbb{P}\!\left(\,|w_i-\E w_i|\ge t\,\right)\ \le\ 2\exp\!\left(-c\kappa\,t^2\right)
\end{align}
for some absolute constant $c>0$ and all $t > 0$.
The desired concentration bound thus can be achieved if the curvature parameter of $V(w)$ satisfies  $\kappa\ge\tilde{\Omega}(d^3).$
To this end, we first observe that this distribution is $\Omega(d)$-strongly log-concave. Indeed, since
\[
V(w)\ =\ -\sum_{1\le i<j\le d}\log\abs{w_i-w_j}\ -\ \sum_{1\le i<j\le d}\log\abs{w_i+w_j},
\]
we get
\[
\nabla^2 V(w)\ =\ \sum_{i<j}\frac{(e_i-e_j)(e_i-e_j)^\top}{(w_i-w_j)^2}
\ +\ \sum_{i<j}\frac{(e_i+e_j)(e_i+e_j)^\top}{(w_i+w_j)^2}\,.
\]
Now $0\le w_i+w_j\le 1$ implies $1/(w_i+w_j)^2 \ge 1$, and hence
\[
\sum_{i<j}\frac{(e_i+e_j)(e_i+e_j)^\top}{(w_i+w_j)^2} \succeq \sum_{i<j}(e_i+e_j)(e_i+e_j)^\top\,.
\]
Note that each non-diagonal entry of $\sum_{i<j}(e_i+e_j)(e_i+e_j)^\top$ is equal to $1$, while each diagonal entry is equal to $d-1$, therefore
\[
\sum_{i<j}(e_i+e_j)(e_i+e_j)^\top = (d-2)\Id_d + \textbf{1}\textbf{1}^\top\,,
\]
where $\textbf{1}$ denotes the all-ones vector. 
For $d \ge 3$, $(d-2)\Id_d \succeq \Omega(d)\cdot \Id_d$, so the distribution is strongly log-concave\footnote{Since we would like to show a high-probability bound (as $d\to \infty$), we can always assume $d\ge 3$. However, if $d=2$, it is also true that $\nabla^2 V(w) \succeq \Omega(1)\cdot \Id_2$.}. 
Since $\sum_{i=1}^d w_i = 1$, by permutation symmetry $\E w_i = 1/d$ for each $i\in [d],$ and by union bound we get $\max_{i\in [d]} w_i \le \tilde{O}(1/\sqrt{d})$ with high probability. Unfortunately, this bound only yields $\norm{Z} \le \tilde{O}(d^{3/2}/\eps n)$, which is significantly worse than what we aimed for. 

To derive a sharper strong convexity bound,  we condition $w$ on the high-probability event $\max_{i\in [d]} w_i \le \tilde{O}(1/\sqrt{d})$. Under this conditioning, for all $i,j\in[d]$, $w_i + w_j \le \tilde{O}(1/\sqrt{d})$ and $1/(w_i + w_j)^2 \ge \tilde{\Omega}(d)$. Crucially, this implies $\tilde{\Omega}(d^2)$-strong convexity of $V(w)$ on the set $\max_{i\in [d]} w_i \le \tilde{O}(1/\sqrt{d})$, which, in turn, implies that $\max_{i\in [d]} w_i \le \tilde{O}(1/d)$ with high probability. This bound yields $\norm{Z} \le \tilde{O}(d/\eps n)$, which is slightly worse than the $O(d/\eps n)$  we need. To get this (optimal) bound, we can repeat the previous argument with the new set: we condition $w$ on the (high-probability) event $\max_{i\in [d]} w_i \le \tilde{O}(1/d)$, and get $\tilde{\Omega}(d^3)$-strong convexity of $V(w)$ on this set, implying now that with high probability $\max_{i\in [d]} w_i \le 1/d + \tilde{O}(1/d^{3/2})$, and leading to the desired bound on $\norm{Z}$.

\paragraph{The projection mechanism.}
The perturbation mechanism in Algorithm~\ref{alg:main-schatten} yields optimal accuracy for every $p$-norm, provided the dataset is sufficiently large. Our approach to handle small datasets is to orthogonally project $\Sigma+Z$ back to the nuclear norm ball of radius $r>0$. Note that, this orthogonal projection may distort $p$-norms for $p\neq 2$ and thus destroys the guarantees we carefully obtained. Nevertheless, if $r\geq \Norm{\Sigma}_{\textnormal{S}_1}=\Tr\Sigma$, it cannot increase the error in Frobenius norm: $\normf{\tilde{\Sigma} - \Sigma} \le \normf{\hat{\Sigma} - \Sigma}$ (since the nuclear norm ball is convex).

Furthermore, a standard strong convexity argument (e.g. the one used in Lemma 5.2 in \cite{cohen2024perturb}) implies that if $r\geq\Tr\Sigma$, then the orthogonal projection $\Pi(\Sigma+Z)$ satisfies
\begin{align}
    \Norm{\Sigma -\Pi(\Sigma+Z)}_{\textnormal{F}}^2 &\leq O\Paren{\iprod{Z, \Sigma - \Pi(\Sigma+Z)}}\nonumber\\
    &\leq \Norm{Z}\cdot O\Paren{\Norm{\Sigma}_{*}+\Norm{\Pi(\Sigma+Z)}_{*}}\nonumber\\
    &\leq O(d/n\eps)\Paren{\Norm{\Sigma}_{*}+\Norm{\Pi(\Sigma+Z)}_{*}}\nonumber\\
    &\leq O(dr/n\eps).\nonumber
\end{align}
where in the before-last step we used the fact that the spectral norm of $Z$ concentrates around $\Theta(d/n\eps)$ with high probability, and in the last step the radius of the nuclear norm ball.
The naive choice $r=1$ would recover the result of \cite{nikolov2023private}. However, this choice can be wasteful. We instead use an $\e$-DP Laplace estimator of $\Tr\Sigma$. Note that with probability at least $1-\beta$, $r \le O\Paren{\Tr(\Sigma) + \log(1/\beta)/(n\varepsilon)}$, and we get the bound as in Theorem \ref{thm:main-frobenius}.

Finally, if $r < \Tr(\Sigma)$, it means that the magnitude of Laplace noise is larger than $\Tr(\Sigma)$. With probability $1-\beta$, it happens only if  $\log(1/\beta)/(n\varepsilon) > \Omega(\Tr(\Sigma))$. In this case, since $r < \Tr(\Sigma) < O\Paren{\log(1/\beta)/(n\varepsilon)}$, both $\tilde{\Sigma}$ and $\Sigma$ are in the same nuclear norm ball of radius $O\Paren{\log(1/\beta)/(n\varepsilon)}$, so the nuclear (and, hence, also Frobenius) distance between them is bounded by $O\Paren{\log(1/\beta)/(n\varepsilon)} \le O\Paren{\sqrt{d}/(n\varepsilon)}$ with probability at least $1-\exp(-\Omega(\sqrt{d}))$, and hence the bound in in Theorem \ref{thm:main-frobenius} is satisfied also in this case.

\section{Future work}\label{sec:future}
It is natural to ask whether a generalization of both Theorem \ref{thm:main-schatten-norm} and Theorem \ref{thm:main-frobenius} is possible. Concretely, it would be interesting to determine if there is an estimator satisfying
    \begin{align*}
        \Norm{\tilde{\Sigma}-\Sigma}_{\textnormal{S}_p}\leq \tilde{O}\Paren{1}\cdot\min\Set{\frac{d^{1+1/p}}{\eps n},\, \;d^{1/p}\sqrt{\frac{\Tr\Sigma}{\eps n}}+\frac{d^{1/p}}{\eps n}}.
    \end{align*}

In addition, it would be interesting to see complementary lower bounds. In particular, currently even in the large sample regime $n\ge d^2/\varepsilon$ it is only known that our result is optimal for $p\ge 2$, but it is not known for $p< 2$. And in the small-sample regime, the picture is even less clear.
\section{Analysis of the additive mechanism}\label{sec:meta-theorem}
In this section we prove Theorem \ref{thm:main-schatten-norm}.

\paragraph{Setup.}
Let $A, A'\in\mathbb{R}^{d\times d}$, $\Delta > 0$. We say that $A$ and $A'$ are nuclear-norm adjacent if
\[
\norm{A-A'}_* \le \Delta\,.
\]
Let
\[
  \mathcal{M}_\rho(A)=A+Z_\rho, \qquad \text{density}(Z_\rho) \propto\ \exp\Paren{-\norm{Z}_*/\rho}.
\]
Write $\norm{\cdot}_{\Schat_p}$ for Schatten-$p$ norms and denote the singular values of $X$ by $\sigma(X)\in\mathbb{R}_+^d$.

\begin{theorem}\label{thm:privacy-nuclear}
If $\rho=\Delta/\varepsilon$, then $\mathcal{M}_\rho$ is $\varepsilon$-differentially private.
\end{theorem}

\begin{proof}
For any $y$ and adjacent $A\sim A'$,
\[
\frac{p_{\mathcal{M}_\rho(A)}(y)}{p_{\mathcal{M}_\rho(A')}(y)}
=\exp\Paren{\frac{\norm{y-A'}_*-\norm{y-A}_*}{\sigma}}
\le \exp\Paren{\frac{\norm{A-A'}_*}{\sigma}}
\le e^{\varepsilon}.
\]
\end{proof}

The next Lemms is a well-known fact about the radial component.

\begin{lemma}[Nuclear-radial factorization]\label{lem:radial}
(Remark 4.2 in \cite{hardt2010geometry})
Let $Z$ have density $p_\rho(Z)\propto \exp\Paren{-\norm{Z}_*/\rho}$. Then
\[
  Z=R\,\Theta,\qquad R:=\norm{Z}_*,\quad \Theta:=Z/\norm{Z}_*\in\mathbb{S}_*:=\{A:\norm{A}_*=1\},
\]
with $R\perp\Theta$ and $R\sim\mathrm{Gamma}\Paren{\text{shape}=d^2,\ \text{scale}=\sigma}$. In particular
\[
  \E R=d^2\sigma,\qquad \E R^2=d^2\Paren{d^2+1}\sigma^2.
\]
\end{lemma}

Now let us study the distribution of the singular spectrum of $Z$. 

\begin{lemma}\label{lem:angular-law}
Let $Z\in\mathbb{R}^{d\times d}$ be distributed with density
\[
  p_\rho(Z)\ \propto\ \exp\Paren{-\norm{Z}_*/\rho},\qquad \sigma>0,
\]
and write an SVD $Z=U\,\mathrm{diag}(\sigma_1,\dots,\sigma_d)\,V^\top$ with $\sigma_i\ge 0$, $U,V\in\mathrm{O}(d)$.
Define the \emph{nuclear radius} $R:=\sum_{i=1}^d \sigma_i=\norm{Z}_*$ and the \emph{normalized singular values}
\[
  w_i \ :=\ \sigma_i/R,\qquad i=1,\dots,d,
\]
so that $w:=(w_1,\dots,w_d)\in\Delta_{d-1}:=\{w\in\mathbb{R}^d_{\ge 0}:\ \sum_i w_i=1\}$, and set
\[
  \Theta\ :=\ \frac{Z}{\norm{Z}_*}\ =\ U\,\mathrm{diag}(w_1,\dots,w_d)\,V^\top.
\]
Then:
\begin{enumerate}
\item[(i)] $R$ is independent of $(U,V,w)$, and $R\sim\mathrm{Gamma}\Paren{\text{shape}=d^2,\ \text{scale}=\rho}$.
\item[(ii)] Conditionally on $(R=r)$, the pair $(U,V)$ is independent of $w$ and is Haar distributed on $\mathrm{O}(d)\times\mathrm{O}(d)$.
\item[(iii)] The conditional law of $w$ given $R=r$ is \emph{independent of $r$} and has density on $\Delta_{d-1}$ (with respect to the $(d-1)$–dimensional Lebesgue measure on the simplex) proportional to
\[
  f_{\Delta}(w)\ \propto\ \prod_{1\le i<j\le d}\abs{w_i^2-w_j^2}\,.
\]
Consequently, the unordered singular values of $\Theta$ have joint density proportional to $f_{\Delta}(w)$ on $\Delta_{d-1}$.
\end{enumerate}
\end{lemma}

\begin{proof}
For real $d\times d$ matrices, Lebesgue measure $dZ$ factors under the SVD map
\[
  \Phi:\ (U,\sigma,V)\ \mapsto\ U\,\mathrm{diag}(\sigma)\,V^\top,\qquad \sigma=(\sigma_1,\dots,\sigma_d)\in\mathbb{R}^d_{\ge 0},
\]
as
\[
  dZ\ =\ C_d\ \underbrace{\prod_{1\le i<j\le d}\abs{\sigma_i^2-\sigma_j^2}}_{\text{Vandermonde in } \sigma^2}\ \,d\sigma\ \,d\mu_{\mathrm{Haar}}(U)\ \,d\mu_{\mathrm{Haar}}(V),
\]
for a constant $C_d>0$, where $d\sigma$ is Lebesgue measure on $\mathbb{R}^d$ restricted to $\sigma_i\ge 0$ and $d\mu_{\mathrm{Haar}}$ denotes Haar probability measure on $\mathrm{O}(d)$ (see Lemma 1.5.3 from \cite{chikuse2003statistics}).

With $dZ\propto \exp\Paren{-\sum_i \sigma_i/\rho}$, the joint density of $(U,\sigma,V)$ is therefore
\[
  \tilde p(U,\sigma,V)\ \propto\ \exp\Paren{-\tfrac{1}{\rho}\sum_{i=1}^d \sigma_i}\ \prod_{i<j}\abs{\sigma_i^2-\sigma_j^2}.
\]

Let $R=\sum_i \sigma_i$ and $w_i=\sigma_i/R$; then $w\in\Delta_{d-1}$ and
the Jacobian is $d\sigma=R^{\,d-1}\, dR\, d\lambda_{\Delta}(w)$.
Moreover,
\[
\prod_{i<j}\lvert \sigma_i^2-\sigma_j^2\rvert
=\prod_{i<j}\lvert (Rw_i)^2-(Rw_j)^2\rvert
= R^{d\cdot(d-1)} \prod_{i<j}\lvert w_i^2-w_j^2\rvert .
\]

Substituting $\sigma_i=Rw_i$ yields the joint density
\[
p(U,R,w,V)\ \propto\ \underbrace{e^{-R/\rho}\,R^{\,d-1}\,R^{d\cdot(d-1)}}_{=\,e^{-R/\rho}\,R^{\,d^2-1}}\;
\cdot\;\underbrace{\prod_{i<j}\lvert w_i^2-w_j^2\rvert}_{=:F(w)}\;
\cdot\; d\mu_{\mathrm{Haar}}(U)\, d\mu_{\mathrm{Haar}}(V)\, dR\, d\lambda_{\Delta}(w).
\]
This factorizes into a function of $R$ times a function of $(U,w,V)$, so:
\[
R\ \perp\ (U,V,w),\qquad f_R(R)\ \propto\ e^{-R/\rho}R^{\,d^2-1},
\]
i.e.\ $R\sim\mathrm{Gamma}(d^2,\rho)$, and given $R$ the pair $(U,V)$ is Haar and independent of $w$, while
$w$ has density proportional to $F(w)$ on $\Delta_{d-1}$, independent of $R$.

Finally, $\Theta=Z/\|Z\|_*=U\,\mathrm{diag}(w)\,V^\top$, so the singular values of $\Theta$ equal $w$.
\end{proof}

Now we are ready to restate and prove Theorem \ref{thm:main-schatten-norm}. Note that it is enough to prove it for spectral norm, for other norm the result follows from the general $\ell_p$-$\ell_\infty$-norm inequality.

\begin{theorem}\label{thm:hp-spectral}
Let $Z\in\mathbb{R}^{d\times d}$ have density proportional to $\exp(-\|Z\|_*/\rho)$.
Write $R:=\|Z\|_*$, $\Theta:=Z/\|Z\|_*$ so that $\|\Theta\|_*=1$, and let
$w=(w_1,\dots,w_d)$ be the (unordered) singular values of $\Theta$ normalized so that $\sum_i w_i=1$.
There exist an absolute constant $C>0$ such that with probability at least $1-2\exp(-d^{1/3}/C)$,
\[
\|Z\| \le 1.1\cdot \rho d\,.
\]
\end{theorem}

\begin{proof}
The density of $w$ is proportional to $e^{-V(w)}$ with
\[
V(w)\ =\ -\sum_{1\le i<j\le d}\log\abs{w_i-w_j}\ -\ \sum_{1\le i<j\le d}\log\abs{w_i+w_j},
\]
restricted to the simplex $\sum_i w_i=1$, $w_i>0$. On the tangent subspace $\{u:\sum_i u_i=0\}$,
\[
\nabla^2 V(w)\ =\ \sum_{i<j}\frac{(e_i-e_j)(e_i-e_j)^\top}{(w_i-w_j)^2}
\ +\ \sum_{i<j}\frac{(e_i+e_j)(e_i+e_j)^\top}{(w_i+w_j)^2}\ \succeq\ (d-2)\,\Id_d.
\]
Since for $d=1$ and $d=2$ the statement of the theorem is always true (for large enough constant $C$), further we consider the case $d\ge 3$. In this case $(d-2)\,\Id_d \succeq 0.1 d$.

If $\max_i w_i\le \alpha$ then $w_i+w_j\le 2\alpha$ for all $i\neq j$, hence
\begin{equation}\label{eq:curv-cap}
\nabla^2 V(w)\big|_{\{\sum_i u_i=0\}}\ \succeq\ \frac{0.1 d}{(2\alpha)^2}\,\Id_d\ \ \text{whenever }\ \max_i w_i\le \alpha.
\end{equation}
Since $w$ is strongly log-concave, for each
$L$–Lipschitz $F$ and each $t>0$,
\begin{equation}\label{eq:lsi}
\mathbb{P}\!\left(\,|F(w)-\E F(w)|\ge t\,\Big|\,\max_i w_i\le \alpha\right)
\ \le\ 2\exp\!\left(-\,c\,\frac{d\,t^2}{(2\alpha)^2\,L^2}\right).
\end{equation}

\medskip\noindent
\paragraph{Step 0 (Bound via global curvature).}
On the full simplex \(\{w_i\ge 0,\ \sum_i w_i=1\}\)
the Hessian on the tangent subspace \(\mathsf T:=\{u:\sum_i u_i=0\}\) satisfies
\[
\nabla^2 V(w)\big|_{\mathsf T}\ \succeq\ 0.1 d \cdot \Id_d\qquad\text{for all }w.
\]
Let \(F(w):=w_i\). Then \(F\) is \(1\)-Lipschitz on \(\mathsf T\), and by strong
log–concavity with curvature $0.1 d$ we have, for some absolute constant $c_1$ and all \(t\ge 0\),
\[
\mathbb{P}\big(|F(w)-\E F(w)| \ge t\big)\ \le\ 2\exp\!\Big(-c_1 d\,t^2\Big).
\]
Note that by permutational symmetry and $\sum_i w_i = 1$, $\E w_i = 1/d$.
Therefore, by union bound, for each \(\beta\in(0,1)\), with probability at least \(1-\beta\),
\[
\max_i w_i \le\ \frac{1}{d}\ + \sqrt{\frac{\log (d/\beta)}{c_1 d}}\,.
\]

Define the set
\[
\cS_1:=\Big\{w:\ \max_i w_i\le \alpha_1\Big\},\qquad \alpha_1:=C_1\sqrt{\frac{\log (1/\beta)}{d}}
\]
for some large enough absolute constant $C_1>0$.
Then $\mathbb{P}(w\in\cS_1)\ge 1-\beta/4$.

\medskip\noindent
\textbf{Step 1 (Boost curvature to $\tilde{\Theta}(d^2)$ and push $\|w\|_\infty$ to $\tilde{O}(1/d)$).}
On $\cS_1$, \eqref{eq:curv-cap} gives curvature
$\kappa_1\ge 0.1 d/(2\alpha_1)^2 \ge \Omega(d^2/\log(1/\beta))$. Apply \eqref{eq:lsi} to the 1–Lipschitz linear functionals
$w\mapsto w_i$ and union–bound over $i=1,\dots,d$:
\[
\mathbb{P}\!\left(\,\max_{1\le i\le d}(w_i-\E w_i)\ \ge\ t\ \Big|\ w\in\cS_1\right)
\ \le\ 2d\,\exp(-c_2\,d^2 t^2/\log(1/\beta)).
\]
Choose $t:= C_2\,{\log(d/\beta)}/d$ with $C_2$ large enough absolute constant, so that the right-hand side is $\le \beta/4$.
Thus, with probability at least $1-\beta/2$,
\begin{equation}\label{eq:stage1}
\max_i w_i\ \le\ \alpha_2\ :=\ \frac{1}{d}\ +\ \frac{C_2{\log(d/\beta)}}{d}\ \le\ \frac{C_2'{\log(d/\beta)}}{d}\,.
\end{equation}

\medskip\noindent
\textbf{Step 2 (Boost curvature to $\tilde{\Theta}(d^3)$ and squeeze $\|w\|_\infty$ to $1/d$ up to $(\log d)/d^{3/2}$).}
Define the tighter cap $\cS_2:=\{w:\max_i w_i\le \alpha_2\}$ with $\alpha_2$ from \eqref{eq:stage1}.
On $\cS_2$, \eqref{eq:curv-cap} yields curvature
\[
\kappa_2\ \ge\ \Omega\Paren{\frac{d^3}{\log^2(d/\beta)}}\,.
\]
Applying \eqref{eq:lsi} as before and union–bounding over $i$ gives, for all $t>0$,
\[
\mathbb{P}\!\left(\,\max_{1\le i\le d}(w_i-1/d)\ \ge\ t\ \Big|\ w\in\cS_2\right)
\ \le\ 2d\,\exp\!\left(-c\,\frac{d^3}{\log^2(d/\beta)}\,t^2\right).
\]
Set $t:= C_3\,\Paren{\frac{\log(d/\beta)}{d}}^{3/2}$; then $(d^3/\log(d/\beta))\,t^2 = C_3^2 \log(d/\beta)$. Choosing $C_3$ large enough so that
$2d\,e^{-c C_3^2 \log(d/\beta)}\le \beta/4$, we get with probability at least $1-3\beta/4$,
\begin{equation}\label{eq:stage2}
\max_i w_i\ \le\ \frac{1}{d}\ +C_3\,\Paren{\frac{\log(d/\beta)}{d}}^{3/2}\,.
\end{equation}

\medskip\noindent
\textbf{Radial concentration.}
Write
\[
  Z=R\,\Theta,\qquad R:=\norm{Z}_*,\qquad \Theta:=Z/\norm{Z}_*\in\mathbb{S}_*:=\{A:\norm{A}_*=1\}.
\]
As shown in Lemma~\ref{lem:radial}, $R\perp \Theta$ and $R\sim \mathrm{Gamma}\Paren{d^2,\rho}$, i.e.\ $R=\rho\sum_{i=1}^{d^2} E_i$ with i.i.d.\ $E_i\sim \mathrm{Exp}(1)$.
By Bernstein inequality for the sum of iid exponential distributions,
for all $\beta\in(0,1/e)$, with probability at least $1-\beta$,
\[
  d^2\rho - C_4d\rho\sqrt{\log (1/\beta)} - C_4\rho\log(1/\beta) \le R \le 
  d^2\rho + C_4d\rho\sqrt{\log (1/\beta)} + C_4\rho\log(1/\beta)\,,
\label{eq:R-highprob}
\]
where $C_4>0$ is an absolute constant.

\medskip\noindent
\textbf{Putting everything together.}
\eqref{eq:stage2} and \eqref{eq:R-highprob} imply that with probability at least $1-\beta$,
\[
\|Z\| \ \le\ \Big(d^2\rho + C_4\,\rho(d\sqrt{\log (1/\beta)}+\log (1/\beta))\Big)\,\Big(\frac{1}{d} + \frac{C_3\,\log^{3/2}(d/\beta)}{d^{3/2}}\Big).
\]
For $\beta = \exp(-d^{1/3}/C)$ where $C>0$ is large enough absolute constant, we get the desired bound
\[
\|Z\| \le 1.1\cdot  \rho d\,.
\]

\end{proof}

\phantomsection
\addcontentsline{toc}{section}{Bibliography}
{\footnotesize
\bibliographystyle{amsalpha} 
\bibliography{scholar}

@book{chikuse2003statistics,
  title={Statistics on special manifolds},
  author={Chikuse, Yasuko},
  volume={174},
  year={2003},
  publisher={Springer Science \& Business Media}
}

@article{amin2019differentially,
  title={Differentially private covariance estimation},
  author={Amin, Kareem and Dick, Travis and Kulesza, Alex and Munoz, Andres and Vassilvitskii, Sergei},
  journal={Advances in Neural Information Processing Systems},
  volume={32},
  year={2019}
}

@inproceedings{nikolov2013geometry,
  title={The geometry of differential privacy: the sparse and approximate cases},
  author={Nikolov, Aleksandar and Talwar, Kunal and Zhang, Li},
  booktitle={Proceedings of the forty-fifth annual ACM symposium on Theory of computing},
  pages={351--360},
  year={2013}
}

@inproceedings{kasiviswanathan2010price,
  title={The price of privately releasing contingency tables and the spectra of random matrices with correlated rows},
  author={Kasiviswanathan, Shiva Prasad and Rudelson, Mark and Smith, Adam and Ullman, Jonathan},
  booktitle={Proceedings of the forty-second ACM symposium on Theory of computing},
  pages={775--784},
  year={2010}
}

@article{dwork2015efficient,
	title={Efficient algorithms for privately releasing marginals via convex relaxations},
	author={Dwork, Cynthia and Nikolov, Aleksandar and Talwar, Kunal},
	journal={Discrete \& Computational Geometry},
	volume={53},
	pages={650--673},
	year={2015},
	publisher={Springer}
}

@article{johnson2010johnson,
  title={The Johnson--Lindenstrauss lemma almost characterizes Hilbert space, but not quite},
  author={Johnson, William B and Naor, Assaf},
  journal={Discrete \& Computational Geometry},
  volume={43},
  number={3},
  pages={542--553},
  year={2010},
  publisher={Springer}
}

@article{brinkman2005impossibility,
  title={On the impossibility of dimension reduction in l1},
  author={Brinkman, Bo and Charikar, Moses},
  journal={Journal of the ACM (JACM)},
  volume={52},
  number={5},
  pages={766--788},
  year={2005},
  publisher={ACM New York, NY, USA}
}

@article{lee2005metric,
  title={Metric structures in L1: dimension, snowflakes, and average distortion},
  author={Lee, James R and Mendel, Manor and Naor, Assaf},
  journal={European Journal of Combinatorics},
  volume={26},
  number={8},
  pages={1180--1190},
  year={2005},
  publisher={Elsevier}
}

@inproceedings{hardt2010geometry,
  title={On the geometry of differential privacy},
  author={Hardt, Moritz and Talwar, Kunal},
  booktitle={Proceedings of the forty-second ACM symposium on Theory of computing},
  pages={705--714},
  year={2010}
}

@inproceedings{cohen2024perturb,
  title={Perturb-and-project: differentially private similarities and marginals},
  author={Cohen-Addad, Vincent and d'Orsi, Tommaso and Epasto, Alessandro and Mirrokni, Vahab and Zhong, Peilin},
  booktitle={Proceedings of the 41st International Conference on Machine Learning},
  pages={9161--9179},
  year={2024}
}

@article{johnson1984extensions,
	title={Extensions of Lipschitz mappings into a Hilbert space},
	author={Johnson, William B},
	journal={Contemp. Math.},
	volume={26},
	pages={189--206},
	year={1984}
}

@inproceedings{dwork2006calibrating,
	title={Calibrating noise to sensitivity in private data analysis},
	author={Dwork, Cynthia and McSherry, Frank and Nissim, Kobbi and Smith, Adam},
	booktitle={Theory of Cryptography: Third Theory of Cryptography Conference, TCC 2006, New York, NY, USA, March 4-7, 2006. Proceedings 3},
	pages={265--284},
	year={2006},
	organization={Springer}
}

@inproceedings{nikolov2023private,
  title={Private query release via the johnson-lindenstrauss transform},
  author={Nikolov, Aleksandar},
  booktitle={Proceedings of the 2023 Annual ACM-SIAM Symposium on Discrete Algorithms (SODA)},
  pages={4982--5002},
  year={2023},
  organization={SIAM}
}

@article{dong2022differentially,
  title={Differentially Private Covariance Revisited},
  author={Dong, Wei and Liang, Yuting and Yi, Ke},
  journal={Advances in Neural Information Processing Systems},
  volume={35},
  pages={850--861},
  year={2022}
}
}

\end{document}